\documentclass{article}
\usepackage{amsmath,amsbsy,amsfonts,amssymb,amsthm,color,dsfont}

\def\ddefloop#1{\ifx\ddefloop#1\else\ddef{#1}\expandafter\ddefloop\fi}

\def\ddef#1{\expandafter\def\csname b#1\endcsname{\ensuremath{\mathbf{#1}}}}
\ddefloop ABCDEFGHIJKLMNOPQRSTUVWXYZ\ddefloop

\def\ddef#1{\expandafter\def\csname bb#1\endcsname{\ensuremath{\mathbb{#1}}}}
\ddefloop ABCDEFGHIJKLMNOPQRSTUVWXYZ\ddefloop

\def\ddef#1{\expandafter\def\csname c#1\endcsname{\ensuremath{\mathcal{#1}}}}
\ddefloop ABCDEFGHIJKLMNOPQRSTUVWXYZ\ddefloop

\def\ddef#1{\expandafter\def\csname v#1\endcsname{\ensuremath{\boldsymbol{#1}}}}
\ddefloop ABCDEFGHIJKLMNOPQRSTUVWXYZabcdefghijklmnopqrstuvwxyz\ddefloop

\def\ddef#1{\expandafter\def\csname v#1\endcsname{\ensuremath{\boldsymbol{\csname #1\endcsname}}}}
\ddefloop {alpha}{beta}{gamma}{delta}{epsilon}{varepsilon}{zeta}{eta}{theta}{vartheta}{iota}{kappa}{lambda}{mu}{nu}{xi}{pi}{varpi}{rho}{varrho}{sigma}{varsigma}{tau}{upsilon}{phi}{varphi}{chi}{psi}{omega}{Gamma}{Delta}{Theta}{Lambda}{Xi}{Pi}{Sigma}{varSigma}{Upsilon}{Phi}{Psi}{Omega}\ddefloop

\newcommand\ind[1]{\ensuremath{\mathds{1}\{#1\}}}

\newtheorem{lemma}{Lemma}

\newtheorem{proposition}{Proposition}
\newtheorem{theorem}{Theorem}
\newtheorem{corollary}{Corollary}
\theoremstyle{remark}
\newtheorem{remark}{Remark}
\theoremstyle{definition}
\newtheorem{definition}{Definition}
\newtheorem{condition}{Condition}

\newcommand\parens[1]{\left( #1 \right)}

\newcommand\braces[1]{\left\{ #1 \right\}}

\usepackage{hyperref}
\usepackage{url}
\usepackage{cite}
\usepackage{fullpage}
\usepackage{algorithm}
\usepackage{algorithmic}
\usepackage{xspace}
\usepackage{authblk}
\usepackage[numbers]{natbib}

\newcommand\algAA{\ensuremath{\textsc{lmm}}\xspace}
\newcommand\algA{\ensuremath{\mathsf{A}}\xspace}
\newcommand\algS{\ensuremath{\mathsf{S}}\xspace}
\newcommand\algM{\ensuremath{\mathsf{M}}\xspace}
\newcommand\algMS{\ensuremath{\mathsf{MS}}\xspace}

\newcommand\EM{\ensuremath{\textsc{EM}}\xspace}

\newcommand\Lap{\ensuremath{\operatorname{Lap}}}

\newcommand\err{\ensuremath{\operatorname{err}}}
\newcommand\haterr{\ensuremath{\widehat{\operatorname{err}}}}

\newcommand{\order}[2]{#1^{(#2)}}

\title{%
  \mbox{The Large Margin Mechanism}
  \mbox{for Differentially Private Maximization}%
}

\author[1]{\mbox{Kamalika Chaudhuri}}
\author[2]{\mbox{Daniel Hsu}}
\author[1]{\mbox{Shuang Song}}
\affil[1]{University of California, San Diego}
\affil[2]{Columbia University}

\begin{document}

\maketitle
{\def\thefootnote{}%
  \footnotetext{E-mail:
  \texttt{kamalika@cs.ucsd.edu},
  \texttt{djhsu@cs.columbia.edu},
  \texttt{shs037@eng.ucsd.edu}}}

\begin{abstract}
  
A basic problem in the design of privacy-preserving algorithms is the
\emph{private maximization problem}: the goal is to pick an item from a
universe that (approximately) maximizes a data-dependent function, all under
the constraint of differential privacy. This problem has been used as a
sub-routine in many privacy-preserving algorithms for statistics and
machine-learning.

Previous algorithms for this problem are either range-dependent---i.e., their
utility diminishes with the size of the universe---or only apply to very
restricted function classes.  This work provides the first general-purpose,
range-independent algorithm for private maximization that guarantees
approximate differential privacy. Its applicability is demonstrated on two
fundamental tasks in data mining and machine learning.

\end{abstract}

\section{Introduction}
\label{sec:intro}

Differential privacy~\cite{DMNS06} is a cryptographically-motivated
definition of privacy that has recently gained significant attention
in the data mining and machine learning communities.
An algorithm for processing sensitive data enforces differential
privacy by ensuring that the likelihood of any outcome does not change
by much when a single individual's private data changes.
Privacy is typically guaranteed by adding noise either to the
sensitive data, or to the output of an algorithm that
processes the sensitive data.
For many machine learning tasks, this leads to a corresponding degradation in
accuracy or utility.
Thus a central challenge in differentially private learning is to
design algorithms with better tradeoffs between privacy and utility for a
wide variety of statistics and machine learning tasks. 

In this paper, we study the \emph{private maximization problem}, a
fundamental problem that arises while designing privacy-preserving
algorithms for a number of statistical and machine learning
applications.
We are given a sensitive dataset $D \subseteq \cX^n$ comprised of
records from $n$ individuals. We are also given a data-dependent objective
function $f : \cU \times \cX^n \to \bbR$, where $\cU$ is a universe of
$K$ items to choose from, and $f(i,\cdot)$ is 
$(1/n)$-Lipschitz for all $i \in \cU$. That is, $|f(i,D') - f(i,D'')| \leq
1/n$ for all $i$ and for any $D', D'' \in \cX^n$ differing in just one individual's
entry.
Always selecting an item that exactly maximizes $f(\cdot,D)$ is generally
non-private, so the goal is to select, in a differentially private
manner, an item $i \in \cU$ with as high an objective $f(i,D)$
as possible.
This is a very general algorithmic problem that arises in many
applications, include private PAC learning~\cite{KLNRS08}
(choosing the most accurate classifier),
private decision tree induction~\cite{FS10}
(choosing the most informative split),
private frequent itemset mining~\cite{BLST10}
(choosing the most frequent itemset),
private validation~\cite{CV13}
(choosing the best tuning parameter),
and private multiple hypothesis testing~\cite{Y13}
(choosing the most likely hypothesis).

The most common algorithms for this problem are the \emph{exponential
mechanism}~\cite{MT07}, and a computationally efficient alternative
from~\cite{BLST10}, which we call the
\emph{max-of-Laplaces mechanism}.
These algorithms are general---they do not require any additional
conditions on $f$ to succeed---and hence have been widely applied.
However, a major limitation of both algorithms is that their utility
suffers from an explicit \emph{range-dependence}: the utility
deteriorates with increasing universe size.
The range-dependence persists even when there is a single clear
maximizer of $f(\cdot,D)$, or a few near maximizers, and even when the
maximizer remains the same after changing the entries of a large
number of individuals in the data.
Getting around range-dependence has therefore been a goal
for designing algorithms for this problem.

This problem has also been addressed by recent algorithms of~\cite{ST13,BNS13},
who provide algorithms that 
are range-independent and satisfy approximate differential privacy, a
relaxed version of differential privacy.
However, none of these algorithms is general; they explicitly fail
unless additional special conditions on $f$ hold.
For example, the algorithm from~\cite{ST13} provides a
range-independent result only when there is a single clear maximizer
$i^*$ such that $f(i^*, D)$ is greater than the second highest value by
some margin; the algorithm from~\cite{BNS13} also has restrictive
conditions that limit its applicability (see
Section~\ref{sec:previous}).
Thus, a challenge is to develop a private
maximization algorithm that is both range-independent and free of
additional conditions; this is necessary to ensure that an algorithm is
widely applicable and provides good utility when the universe size is large.

In this work, we provide the first such general purpose range-independent
private maximization algorithm.
Our algorithm is based on two key insights.
The first is that private maximization is easier when there is a small
set of near maximizing items $j \in \cU$ for which $f(j, D)$ is close
to the maximum value $\max_{i \in \cU} f(i, D)$.
A plausible algorithm based on this insight is to first find a set of
near maximizers, and then run the exponential mechanism on this set.
However, finding this set directly in a differentially private manner
is very challenging.
Our second insight is that only the number $\ell$ of near maximizers needs to
be found in a differentially private manner -- a task that is considerably
easier. Provided there is a margin between
the maximum value and the $(\ell+1)$-th maximum value of $f(i, D)$,
running the exponential mechanism on the items with the top $\ell$ values of $f(i, D)$ 
results in approximate differential privacy as well as good utility.

Our algorithm, which we call the \emph{large margin mechanism}, automatically
exploits large margins when they exist to simultaneously (i) satisfy
approximate differential privacy (Theorem~\ref{thm:privacy}), as well as (ii)
provide a utility guarantee that depends (logarithmically) only on the number
of near maximizers, rather than the universe size (Theorem~\ref{thm:utility}).
We complement our algorithm with a lower bound, showing that the utility of any
approximate differentially private algorithm must deteriorate with the number
of near maximizers (Theorem~\ref{thm:lb2}).  A consequence of our lower bound
is that range-independence cannot be achieved with pure differential privacy
(Proposition~\ref{prop:lb1}), which justifies our relaxation to approximate
differential privacy.

Finally, we show the applicability of our algorithm to two problems from data
mining and machine learning: frequent itemset mining and private PAC learning.
For the first problem, an application of our method gives the first algorithm
for frequent itemset mining that simultaneously guarantees approximate
differential privacy and utility independent of the itemset universe
size.  For the second problem, our algorithm achieves tight sample complexity
bounds for private PAC learning analogous to the shell bounds of~\cite{shell}
for non-private learning.

\section{Background}
\label{sec:prelims}

This section reviews differential privacy and introduces the
private maximization problem.

\subsection{Definitions of Differential Privacy and Private
Maximization}

For the rest of the paper, we consider randomized algorithms $\cA : \cX^n \to \Delta(\cS)$
that take as input datasets $D \in \cX^n$ comprised of records from
$n$ individuals, and output values in a range $\cS$.
Two datasets $D, D' \in \cX^n$ are said to be \emph{neighbors} if they
differ in a single individual's entry.
A function $\phi : \cX^n \to \bbR$ is $L$-Lipschitz if $|\phi(D) -
\phi(D')| \leq L$ for all neighbors $D, D' \in \cX^n$.

The following definitions of (approximate) differential privacy are
from \cite{DMNS06} and \cite{Dwork06}.
\begin{definition}[Differential Privacy]
  A randomized algorithm $\cA:\cX^n\to\Delta(\cS)$ is said to be
  \emph{$(\alpha,\delta)$-approximate differentially private} if, for all neighbors $D, D'
  \in \cX^n$ and all $S \subseteq \cS$,
  \[ \Pr(\cA(D) \in S) \leq e^{\alpha} \Pr(\cA(D') \in S) + \delta . \]
  The algorithm $\cA$ is \emph{$\alpha$-differentially private} if it
  is $(\alpha,0)$-approximate differentially private.
\end{definition}
Smaller values of the privacy parameters $\alpha > 0$ and $\delta \in
[0,1]$ imply stronger guarantees of privacy.

\begin{definition}[Private Maximization]
  In the \emph{private maximization problem}, a sensitive dataset $D
  \subseteq \cX^n$ comprised of records from $n$ individuals is given
  as input; there is also a universe $\cU := \{ 1, \ldots, K \}$ of
  $K$ items, and a function $f : \cU \times \cX^n \to \bbR$ such that
  $f(i,\cdot)$ is $(1/n)$-Lipschitz for all $i \in \cU$.
  The goal is to return an item $i \in \cU$ such that $f(i,D)$ is as
  large as possible while satisfying (approximate) differential
  privacy.
\end{definition}

Always returning the exact maximizer of $f(\cdot,D)$ is non-private,
as changing a single individuals' private values can potentially change the
maximizer. Our goal is to design a randomized algorithm that outputs an
approximate maximizer with high probability.
(We loosely refer to the expected $f(\cdot,D)$ value of the chosen item
as the utility of the algorithm.)

Note that this problem is different from private release of the
\emph{maximum value} of $f(\cdot, D)$; a solution for the latter is
easily obtained by adding Laplace noise with standard deviation
$O(1/(\alpha n))$ to $\max_{i \in \cU} f(i, D)$~\cite{DMNS06}.
Privately returning a nearly maximizing item itself is much more
challenging. 

Private maximization is a core problem in the
design of differentially private algorithms, and arises in numerous
statistical and machine learning tasks.
The examples of frequent itemset mining and PAC learning are discussed
in Sections~\ref{sec:fim} and~\ref{sec:pac}.

\subsection{Previous Algorithms for Private Maximization}
\label{sec:previous}

The standard algorithm for private maximization is the \emph{exponential
mechanism}~\cite{MT07}.
Given a privacy parameter $\alpha>0$, the exponential mechanism
randomly draws an item $i \in U$ with probability $p_i \propto e^{n
\alpha f(i, D)/2}$; this guarantees $\alpha$-differential privacy. 
While the exponential mechanism is widely used because of its
generality, a major limitation is its \emph{range-dependence}---i.e.,
its utility diminishes with the universe size $K$.
To be more precise, consider the following example where $\cX := \cU = [K]$
and
\begin{equation}
  \label{eqn:ex1}
  f(i, D) := \frac1n \left|\braces{ j \in [n] : D_j \geq i }\right|
\end{equation}
(where $D_j$ is the $j$-th entry in the dataset $D$).
When $D = (1,1,\dotsc,1)$, there is a clear maximizer $i^* = 1$, which
only changes when the entries of at least $n/2$ individuals in $D$
change.
It stands to reason that any algorithm should report $i=1$ in this
case with high probability.
However, the exponential mechanism outputs $i=1$ only with probability
$e^{n\alpha/2}/(K - 1 + e^{n \alpha/2})$, which is small unless $n =
\Omega(\log(K)/\alpha)$.
This implies that the utility of the exponential mechanism
deteriorates with $K$.

Another general purpose algorithm is the \emph{max-of-Laplaces}
mechanism from~\cite{BLST10}.
Unfortunately, this algorithm is also range-dependent.
Indeed, our first observation is that all $\alpha$-differentially
private algorithms that succeed on a wide class of private
maximization problems share this same drawback.

\begin{proposition}[Lower bound for differential privacy]
  \label{prop:lb1}
  Let $\cA$ be any $\alpha$-differentially private algorithm for
  private maximization, $\alpha \in (0,1)$, and $n \geq 2$.
  There exists
  a domain $\cX$,
  a function $f : \cU \times \cX^n \to \bbR$ such that $f(i,\cdot)$ is
  $(1/n)$-Lipschitz for all $i \in \cU$,
  and a dataset $D \in \cX^n$ such that:
  \[
    \Pr\parens{ f(\cA(D), D) > \max_{i \in \cU} f(i, D) - \frac{\log
    \frac{K-1}{2}}{\alpha n} } < \frac12
    .
  \]
\end{proposition}
We remark that results similar to Proposition~\ref{prop:lb1} have appeared
in~\cite{HT09,BKN10,CH11,CH12,BLR08}; we simply re-frame those results
here in the context of private maximization.

Proposition~\ref{prop:lb1} implies that in order to remove
range-dependence, we need to relax the privacy notion.
We consider a relaxation of the privacy constraint to $(\alpha,
\delta)$-approximate differential privacy with $\delta>0$.

The approximate differentially private algorithm from~\cite{ST13}
applies in the case where there is a single clear maximizer whose
value is much larger than that of the rest.
This algorithm adds Laplace noise with standard deviation $O(1/(\alpha
n))$ to the difference between the largest and the second-largest
values of $f(\cdot, D)$, and outputs the maximizer if this noisy
difference is larger than $O(\log (1/\delta)/(\alpha n))$; otherwise,
it outputs \texttt{Fail}.
Although this solution has high utility for the example
in~\eqref{eqn:ex1} with $D = (1,1,\dotsc,1)$, it fails even when there
is a single additional item $j \in \cU$ with $f(j, D)$ close to the
maximum value; for instance, $D = (2,2,\dotsc,2)$.

\cite{BNS13} provides an approximate differentially private algorithm that
applies when $f$ satisfies a condition called \emph{$\ell$-bounded growth}.
This condition entails the following: first, for any $i \in \cU$, adding a
single individual to any dataset $D$ can either keep $f(i, D)$ constant, or
increase it by $1/n$; and second, $f(i,D)$ can only increase in this case for
at most $\ell$ items $i \in \cU$.  The utility of this algorithm depends only
on $\log \ell$, rather than $\log K$.  In contrast, our algorithm does not
require the first condition.  Furthermore, to ensure that our algorithm only
depends on $\log \ell$, it suffices that there only be ${\leq}\ell$ near
maximizers, which is substantially less restrictive than the $\ell$-bounded
growth condition.

As mentioned earlier, we avoid range-dependence with an algorithm
that finds and optimizes over {\em{near maximizers}} of $f(\cdot,D)$.
We next specify what we mean by near maximizers using a notion of
\emph{margin}.

\section{The Large Margin Mechanism}

We now our new algorithm for private maximization,
called the \emph{large margin mechanism}, along with its privacy and
 utility guarantees.

\subsection{Margins}

We first introduce the notion of \emph{margin} on which our algorithm
is based.
Given an instance of the private maximization problem and a positive
integer $\ell \in \bbN$, let $\order f\ell(D)$ denote the $\ell$-th
highest value of $f(\cdot, D)$.
We adopt the convention that $\order{f}{K+1}(D) = -\infty$.

\begin{condition}[($\ell,\gamma$)-margin condition]
  \label{cond:margin}
  For any $\ell \in \bbN$ and $\gamma>0$, we say a dataset $D \in \cX^n$
  satisfies the \emph{$(\ell,\gamma)$-margin condition} if
  \[
    \order{f}{\ell+1}(D) < \order{f}{1}(D) - \gamma
  \]
  (\emph{i.e.}, there are at most $\ell$ items within $\gamma$ of the top
  item according to $f(\cdot,D)$).\footnote{%
    Our notion of margins here is different from the usual notion of
    margins from statistical learning that underlies linear prediction
    methods like support vector machines and boosting.
    In fact, our notion is more closely related to the shell
    decomposition bounds of~\cite{shell}, which we discuss in
    Section~\ref{sec:pac}.%
  }
\end{condition}

By convention, every dataset satisfies the $(K,\gamma)$-margin
condition.
Intuitively, a $(\ell,\gamma)$-margin condition with a relatively
large $\gamma$ implies that there are ${\leq}\ell$ near
maximizers, so the private maximization problem is easier when $D$
satisfies an $(\ell,\gamma)$-margin condition with small $\ell$.

How large should $\gamma$ be for a given $\ell$?
The following lower bound suggests that in order to have $n =
O(\log(\ell)/\alpha)$, we need $\gamma$ to be roughly
$\log(\ell)/(\alpha n)$.

\begin{theorem}[Lower bound for approximate differential privacy]
  \label{thm:lb2}
  Fix any $\alpha \in (0,1)$, $\ell > 1$, and $\delta \in
  [0,(1-\exp(-\alpha))/(2(\ell-1))]$; and assume $n \geq
  2$.
  Let $\cA$ be any $(\alpha,\delta)$-approximate differentially
  private algorithm, and $\gamma := \min\{1/2,\
  \log((\ell-1)/2)/(n\alpha)\}$.
  There exists
  a domain $\cX$,
  a function $f : \cU \times \cX^n \to \bbR$ such that $f(i,\cdot)$ is
  $(1/n)$-Lipschitz for all $i \in \cU$,
  and a dataset $D \in \cX^n$ such that:
  \begin{enumerate}
    \item $D$ satisfies the $\parens{\ell, \gamma}$-margin condition.

    \item $\displaystyle\Pr\parens{ f(\cA(D), D) > \order f1(D) -
      \gamma } < \frac{1}{2}$.

  \end{enumerate}
\end{theorem}

A consequence of Theorem~\ref{thm:lb2} is that complete range-independence for
all $(1/n)$-Lipschitz functions $f$ is not possible, even with approximate
differential privacy.  For instance, if $D$ satisfies an
$(\ell,\Omega(\log(\ell)/(\alpha n)))$-margin condition only when $\ell =
\Omega(K)$, then $n$ must be $\Omega(\log(K)/\alpha)$ in order for an
approximate differentially private algorithm to be useful.

\subsection{Algorithm}

The lower bound in Theorem~\ref{thm:lb2} suggests the following
algorithm.
First, privately determine a pair $(\ell,\gamma)$, with $\ell$ is as
small as possible and $\gamma = \Omega(\log(\ell)/(\alpha n))$, such
that $D$ satisfies the $(\ell, \gamma)$-margin condition.
Then, run the exponential mechanism on the set $\cU_\ell \subseteq \cU$
of items with the $\ell$ highest $f(\cdot,D)$ values.
This sounds rather natural and simple, but a knee-jerk reaction to
this approach is that the set $\cU_\ell$ itself depends on the
sensitive dataset $D$, and it may have \emph{high sensitivity} in the
sense that membership of many items in $\cU_\ell$ can change when a
single individual's private value is changed.
Thus differentially private computation of $\cU_\ell$ appears
challenging.

It turns out we do not need to guarantee the privacy of the set $\cU_\ell$, but
rather just of a valid $(\ell,\gamma)$ pair.  This is essentially because when
$D$ satisfies the $(\ell, \gamma)$-margin condition, the probability that the
exponential mechanism picks an item $i$ that occurs in $\cU_{\ell}$ when the
sensitive dataset is $D$ but not in $\cU_{\ell}$ when the sensitive dataset is
its neighbor $D'$ is very small.

Moreover, we can find such a valid $(\ell,\gamma)$ pair using a differentially
private search procedure based on the \emph{sparse vector
technique}~\cite{SVT}.  Combining these ideas gives a general (and adaptive)
algorithm whose loss of utility due to privacy is only
$O(\log(\ell/\delta)/\alpha n)$ when the dataset satisfies a $(\ell,
O(\log(\ell/\delta)/(\alpha n))$-margin condition. We call this general
algorithm the \emph{large margin mechanism} (Algorithm~\ref{alg:aa}), or \algAA
for short.

\begin{algorithm}[t]
  \caption{The large margin mechanism $\algAA(
    \alpha,\delta,
  D)$}
  \label{alg:aa}
  \begin{algorithmic}[1]
    \renewcommand{\algorithmicrequire}{\textbf{input}}
    \renewcommand{\algorithmicensure}{\textbf{output}}

    \REQUIRE Privacy parameters
    $\alpha > 0$ and $\delta \in (0,1)$,
    database $D \in \cX^n$.

    \ENSURE Item $I \in \cU$.

    \STATE For each $r = 1,2,\dotsc,K$, let
%
      \begin{align*}
        \order tr
        & := \frac{6}{n}
        \parens{ 1 + \frac{\ln(3r/\delta)}{\alpha} }
        = O\parens{ \frac1n + \frac{1}{n\alpha} \log \frac{r}{\delta} }
        , \\
        \order Tr
        & := \frac{3}{n\alpha} \ln \frac{3}{2\delta}
        + \frac{6}{n\alpha} \ln \frac{3}{\delta}
        + \frac{12}{n\alpha} \ln \frac{3r(r+1)}{\delta}
        + \order tr
        = O\parens{ \frac1n + \frac{1}{n\alpha} \log \frac{r}{\delta} }
        .
      \end{align*}

    \STATE Draw $Z \sim \Lap(3/\alpha)$.

    \STATE Let $m := \order f1(D) + Z/n$.
    \label{step:m}
    \COMMENT{Estimate of $\max$ value.}

    \STATE Draw $G \sim \Lap(6/\alpha)$ and $Z_1, Z_2, \dotsc,
    Z_{K-1} \stackrel{\text{iid}}{\sim} \Lap(12/\alpha)$.

    \STATE Let $\ell := 1$.
    \COMMENT{Adaptively determine value $\ell$ such that $D$ satisfies
    $(\ell,\order t\ell)$-margin condition.}

    \WHILE{$\ell < K$}

      \IF{$m - \order f{\ell+1}(D) > (Z_\ell + G)/n + \order T\ell$}

        \STATE Break out of while-loop with current value of $\ell$.

      \ELSE

        \STATE Let $\ell := \ell + 1$.

      \ENDIF

    \ENDWHILE
    \label{step:ell}

    \STATE Let $\cU_\ell$ be the set of $\ell$ items in $\cU$ with
    highest $f(i,D)$ value (ties broken arbitrarily).

    \STATE Draw $I \sim \vp$ where $p_i \propto \ind{i \in \cU_\ell}
    \exp(n\alpha f(i,D)/6)$.
    \COMMENT{Exponential mechanism on top $\ell$ items.}
    \label{step:p}

%

    \RETURN $I$.
    \label{step:I}

  \end{algorithmic}
\end{algorithm}

\subsection{Privacy and Utility Guarantees}

We first show that \algAA satisfies approximate differential privacy.

\begin{theorem}[Privacy guarantee]
  \label{thm:privacy}
  $\algAA(\alpha,\delta,\cdot)$
  satisfies $(\alpha,\delta)$-approximate
  differential privacy.
\end{theorem}

The proof of Theorem~\ref{thm:privacy} is in Appendix~\ref{app:privacy}.  The
following theorem, proved in Appendix~\ref{app:utility}, provides a guarantee
on the utility of \algAA.

\begin{theorem}[Utility guarantee]
  \label{thm:utility}
  Pick any $\eta \in (0,1)$.
  Suppose $D \in \cX^n$ satisfies the $(\ell^*,\gamma^*)$-margin condition
  with
  \[
    \gamma^* = \frac{21}{n\alpha}
    \ln \frac{3}{\eta} + \order T{\ell^*}
    .
  \]
  Then with probability at least $1-\eta$, $I :=
  \algAA(\alpha,\delta,D)$
  satisfies
  \[
    f(I,D) \geq \order f1(D) - \frac{6\ln(2\ell^*/\eta)}{n\alpha} .
  \]
  (Above, $\order T{\ell^*}$ is as defined in
  Algorithm~\ref{alg:aa}.)
\end{theorem}
\begin{remark}
  \label{remark:utility}
  Fix some $\alpha,\delta \in (0,1)$.
  Theorem~\ref{thm:utility} states that if the dataset $D$ satisfies
  the $(\ell^*,\gamma^*)$-margin condition, for some positive integer
  $\ell^*$ and $\gamma^* = C \log(\ell^*/\delta) / (n\alpha) $ for
  some universal constant $C>0$, then the value $f(I,D)$ of the item
  $I$ returned by $\algAA$ is within $O( \log(\ell^*) / (n\alpha))$ of
  the maximum, with high probability.
  There is no explicit dependence on the cardinality $K$ of the universe
  $\cU$.
\end{remark}

\section{Illustrative Applications}
\label{sec:applications}

We now describe applications of \algAA to problems from data mining
and machine learning.

\subsection{Private Frequent Itemset Mining}
\label{sec:fim}

Frequent Itemset Mining (FIM) is the following popular data mining problem:
given the purchase lists of users (say, for an online grocery store), the
goal is to find the sets of items that are purchased together most often.
The work of~\cite{BLST10} provides the first differentially private algorithms
for FIM.  However, as these algorithms rely on the exponential mechanism and
the max-of-Laplaces mechanism, their utilities degrade with the total number of
possible itemsets.  Subsequent algorithms exploit other properties of itemsets
or avoid directly finding the most frequent
itemset~\cite{ondpfim12,Privbasis12,PubSet11,FreqPattern13}. 

Let $\cI$ be the set of items that can be purchased, and let $B$ be
the maximum length of an user's purchase list.
Let $\cU \subseteq 2^\cI$ be the family of itemsets of interest.
For simplicity, we let $\cU := \binom{\cI}{r}$---i.e., all itemsets of
size $r$---and consider the problem of picking the itemset with
the (approximately) highest frequency.
This is a private maximization problem where $D$ is the users' lists
of purchased items, and $f(i,D)$ is the fraction of users who purchase
an itemset $i \in \cU$.
Let $f_{\max}$ be the highest frequency of an itemset in $D$.
Let $L$ be the total number of itemsets with non-zero frequency, so $L
\leq n \binom{B}{r}$, which is $\ll |\cI|^r$ whenever $B \ll |\cI|$.
Applying \algAA gives the following guarantee.

\begin{corollary}
  \label{cor:fim}
  Suppose we use $\algAA(\alpha,\delta,\cdot)$ on the FIM problem above.
  Then there exists a constant $C>0$ such that the following holds.
  If $f_{\max} \geq C \cdot \log(L/\delta)/(n \alpha)$, then with
  probability $\geq 1 - \delta$, the frequency of the itemset
  $I_{\algAA}$
  output by \algAA is
  \[
    f(I_{\algAA}, D) \geq f_{\max} - O\parens{ \frac{\log(L/\delta)}{n
    \alpha} }
    .
  \]
\end{corollary}
In contrast, the itemset $I_{\EM}$ returned by the exponential
mechanism is only guaranteed to satisfy
\[
  f(I_{\EM}, D) \geq f_{\max} - O\parens{ \frac{r \log ( |\cI|
  /\delta)}{n \alpha}  }
  ,
\]
which is significantly worse than Corollary~\ref{cor:fim} whenever $L
\ll |\cI|^r$ (as is typically the case).
Second, to ensure differential privacy by running the exponential
mechanism, one needs \emph{a priori} knowledge of the set $\cU$ (and
thus the universe of items $\cI$) independently of the observed
data; otherwise the process will not be end-to-end differentially
private.
In contrast, our algorithm does not need to know $\cI$ in order to
provide end-to-end differential privacy.
Finally, unlike~\cite{ST13}, our algorithm does not require a gap
between the top two itemset frequencies.

\subsection{Private PAC Learning}
\label{sec:pac}

We now consider private PAC learning with a finite hypothesis class $\cH$ with
bounded VC dimension $d$~\cite{KLNRS08}.  Here, the dataset $D$ consists of $n$
labeled training examples drawn iid from a fixed distribution.  The error
$\err(h)$ of a hypothesis $h \in \cH$ is the probability that it misclassifies
a random example drawn from the same distribution.  The goal is to return a
hypothesis $h \in \cH$ with error as low as possible.  A standard procedure
that has been well-studied in the literature simply returns the minimizer
$\hat{h} \in \cH$ of the empirical error $\haterr(h,D)$ computed on the
training data $D$, but this does not guarantee (approximate) differential
privacy.  The work of~\cite{KLNRS08} instead uses the exponential mechanism to
select a hypothesis $h_{\EM} \in \cH$.
With probability $\geq 1 - \delta_0$,
\begin{equation}
  \label{eq:pac-em}
  \err(h_{\EM}) \leq \min_{h \in \cH} \err(h)
  + O\parens{
    \sqrt{\frac{d \log (n/\delta_0)}{n}}
    + \frac{\log |\cH| + \log(1/\delta_0)}{\alpha n}
  }
  .
\end{equation}
The dependence on $\log|\cH|$ is improved to $d\log|\Sigma|$
by~\cite{BLR08} when the data entries come from a finite set $\Sigma$.
The subsequent work of~\cite{ITCS12} introduces the notion of
representation dimension, and shows how it relates to differentially
private learning in the discrete and finite case, and \cite{BNS13}
provides improved convergence bounds with approximate differential
privacy that exploit the structure of some specific hypothesis
classes.
For the case of infinite hypothesis classes and continuous data
distributions,~\cite{CH11} shows that distribution-free
private PAC learning is not generally possible, but
distribution-dependent learning can be achieved under certain
conditions.

We provide a sample complexity bound of a rather different character
compared to previous work.
Our bound only relies on uniform convergence properties of $\cH$, and
can be significantly tighter than the bounds from~\cite{KLNRS08} when
the number of hypotheses with error close to $\min_{h \in \cH}
\err(h)$ is small.
Indeed, the bounds are a private analogue of the \emph{shell bounds}
of~\cite{shell}, which characterize the structure of the hypothesis
class as a function of the properties of a decomposition based on
hypotheses' error rates.
In many situation, these bounds are significantly tighter than those
that do not involve the error distributions.

Following~\cite{shell}, we divide the hypothesis class $\cH$ into $R =
O(\sqrt{n/(d \log n)})$ shells; the $t$-th shell $\cH(t)$ is defined by
\[
  \cH(t)
  := \braces{ h \in \cH : \err(h) \leq \min_{h' \in \cH} \err(h')
  + C_0 t \sqrt{\frac{d \log (n/\delta_0)}{n}} }
  .
\]
Above, $C_0>0$ is the constant from uniform convergence bounds---i.e.,
$C_0$ is the smallest $c>0$ such that for all $h \in \cH$, with
probability $\geq 1 - \delta_0$, we have $|\haterr(h,D) - \err(h)|
\leq c\sqrt{d \log(n/\delta_0)/n}$.
Observe that $\cH(t+1) \subseteq \cH(t)$; and moreover, with
probability $\geq 1 - \delta_0$, all $h \in \cH(t)$ have $\haterr(h,D)
\leq \min_{h' \in \cH} \err(h') + C_0 \cdot (t+1) \sqrt{d
\log(n/\delta_0)/n}$. 

Let $t^*(n)$ as the smallest integer $t \in \bbN$ such that
\[
  \frac{\log(|\cH(t+1)|) + \log(1/\delta)}{t} \leq \frac{C_0 \alpha
  \sqrt{d n \log n}}{C}
\]
where $C>0$ is the constant from Remark~\ref{remark:utility}.
Then, with probability $\geq 1 - \delta_0$, the dataset $D$ with $f =
1 - \haterr$ satisfies the $(\ell,\gamma)$-margin condition, with
$\ell = |\cH(t^*(n)+1)|$ and $\gamma = C \log(|\cH(t^*(n)+1)|/\delta)
/ (\alpha n)$.
Therefore, we have the following guarantee for applying \algAA to this
problem.

\begin{corollary}
  Suppose we use $\algAA(\alpha,\delta,\cdot)$ on the
  learning problem above (with $\cU = \cH$ and $f = 1-\haterr$).
  Then, with probability $\geq 1 - \delta_0 - \delta$, the hypothesis
  $h_{\algAA}$ returned by \algAA satisfies
  \[
    \err(h_{\algAA}) \leq \min_{h \in \cH} \err(h) + O\parens{
      \sqrt{\frac{d \log (n/\delta_0)}{n}}
      + \frac{\log (|\cH(t^*(n) + 1)|/\delta)}{\alpha n}
    }
    .
  \]
\end{corollary}
The dependence on $\log|\cH|$ from~\eqref{eq:pac-em} is replaced here
by $\log(|\cH(t^*(n)+1)|/\delta)$, which can be vastly smaller, as
discussed in~\cite{shell}.

\section{Additional Related Work}
\label{sec:relwork}

There has been a large amount of work on differential privacy for a wide
range of statistical and machine learning tasks over the last
decade~\cite{BDMN05,S11,CMS11,FS10,WZ10,JKT12,BST14}; for overviews, see~\cite{DS08} and~\cite{SC13}.
In particular, algorithms for the private maximization problem (and
variants) have been used as subroutines in many applications; examples
include PAC learning~\cite{KLNRS08}, principle component
analysis~\cite{CSS12}, performance validation~\cite{CV13}, and
multiple hypothesis testing~\cite{Y13}.


A separation between pure and approximate differential privacy has
been shown in several previous works~\cite{DL09, ST13, BNS13}.
The first approximate differentially private algorithm that achieves
a separation is the Propose-Test-Release (PTR)
framework~\cite{DL09}.
Given a function, PTR determines an upper bound on its local
sensitivity at the input dataset through a search procedure; noise
proportional to this upper bound is then added to the actual function
value.
We note that the PTR framework does not directly apply to our setting
as the sensitivity is not generally defined for a discrete universe.

In the context of private PAC learning, the work of \cite{BNS13} gives
the first separation between pure and approximate differential
privacy.
In addition to using the algorithm from \cite{ST13}, they devise two
additional algorithmic techniques: a concave maximization procedure
for learning intervals, and an algorithm for the private maximization
problem under the $\ell$-bounded growth condition discussed in
Section~\ref{sec:previous}.
The first algorithm is specific to their problem and does not appear
to apply to general private maximization problems.
The second algorithm has a sample complexity bound of $n =
O(\log(\ell)/\alpha)$ when the function $f$ satisfies the
$\ell$-bounded growth condition.

Lower bounds for approximate differential privacy have been shown
by~\cite{BLR08,De11,CH12,V14}, and the proof of our
Theorem~\ref{thm:lb2} borrows some techniques from~\cite{CH12}.

\section{Conclusion and Future Work}

In this paper, we have presented the first general and
range-independent algorithm for approximate differentially private
maximization.
The algorithm automatically adapts to the available large margin
properties of the sensitive dataset, and reverts to worst-case
guarantees when such properties are lacking.
We have illustrated the applicability of the algorithm in two
fundamental problems from data mining and machine learning; in future
work, we plan to study other applications where range-independence is
a substantial boon.

\paragraph{Acknowledgments.}
We thank an anonymous reviewer for suggesting the simpler variant of
$\algAA$ based on the exponential mechanism.
(The original version of $\algAA$ used a \emph{max of truncated exponentials} mechanism, which gives the same guarantees up
to constant factors.)
This work was supported in part by the NIH under U54 HL108460 and the
NSF under IIS 1253942.

\bibliographystyle{plainnat}
\bibliography{ref}

\begin{thebibliography}{34}
\providecommand{\natexlab}[1]{#1}
\providecommand{\url}[1]{\texttt{#1}}
\expandafter\ifx\csname urlstyle\endcsname\relax
  \providecommand{\doi}[1]{doi: #1}\else
  \providecommand{\doi}{doi: \begingroup \urlstyle{rm}\Url}\fi

\bibitem[Bassily et~al.(2014)Bassily, Smith, and Thakurta]{BST14}
Raef Bassily, Adam Smith, and Abhradeep Thakurta.
\newblock Private empirical risk minimization, revisited.
\newblock arXiv:1405.7085, 2014.

\bibitem[Beimel et~al.(2010)Beimel, Kasiviswanathan, and Nissim]{BKN10}
Amos Beimel, Shiva~Prasad Kasiviswanathan, and Kobbi Nissim.
\newblock Bounds on the sample complexity for private learning and private data
  release.
\newblock In \emph{Theory of Cryptography}, pages 437--454. Springer, 2010.

\bibitem[Beimel et~al.(2013{\natexlab{a}})Beimel, Nissim, and Stemmer]{BNS13}
Amos Beimel, Kobbi Nissim, and Uri Stemmer.
\newblock Private learning and sanitization: Pure vs. approximate differential
  privacy.
\newblock In \emph{RANDOM}, 2013{\natexlab{a}}.

\bibitem[Beimel et~al.(2013{\natexlab{b}})Beimel, Nissim, and Stemmer]{ITCS12}
Amos Beimel, Kobbi Nissim, and Uri Stemmer.
\newblock Characterizing the sample complexity of private learners.
\newblock In \emph{ITCS}, pages 97--110, 2013{\natexlab{b}}.

\bibitem[Bhaskar et~al.(2010)Bhaskar, Laxman, Smith, and Thakurta]{BLST10}
Raghav Bhaskar, Srivatsan Laxman, Adam Smith, and Abhradeep Thakurta.
\newblock Discovering frequent patterns in sensitive data.
\newblock In \emph{KDD}, 2010.

\bibitem[Blum et~al.(2005)Blum, Dwork, McSherry, and Nissim]{BDMN05}
A.~Blum, C.~Dwork, F.~McSherry, and K.~Nissim.
\newblock Practical privacy: the {SuLQ} framework.
\newblock In \emph{PODS}, 2005.

\bibitem[Blum et~al.(2013)Blum, Ligett, and Roth]{BLR08}
Avrim Blum, Katrina Ligett, and Aaron Roth.
\newblock A learning theory approach to noninteractive database privacy.
\newblock \emph{Journal of the ACM}, 60\penalty0 (2):\penalty0 12, 2013.

\bibitem[Bonomi and Xiong(2013)]{FreqPattern13}
Luca Bonomi and Li~Xiong.
\newblock Mining frequent patterns with differential privacy.
\newblock \emph{Proceedings of the VLDB Endowment}, 6\penalty0 (12):\penalty0
  1422--1427, 2013.

\bibitem[Bun et~al.(2014)Bun, Ullman, and Vadhan]{V14}
Mark Bun, Jonathan Ullman, and Salil Vadhan.
\newblock Fingerprinting codes and the price of approximate differential
  privacy.
\newblock In \emph{STOC}, 2014.

\bibitem[Chaudhuri and Hsu(2011)]{CH11}
Kamalika Chaudhuri and Daniel Hsu.
\newblock Sample complexity bounds for differentially private learning.
\newblock In \emph{COLT}, 2011.

\bibitem[Chaudhuri and Hsu(2012)]{CH12}
Kamalika Chaudhuri and Daniel Hsu.
\newblock Convergence rates for differentially private statistical estimation.
\newblock In \emph{ICML}, 2012.

\bibitem[Chaudhuri and Vinterbo(2013)]{CV13}
Kamalika Chaudhuri and Staal~A Vinterbo.
\newblock A stability-based validation procedure for differentially private
  machine learning.
\newblock In \emph{Advances in Neural Information Processing Systems}, pages
  2652--2660, 2013.

\bibitem[Chaudhuri et~al.(2011)Chaudhuri, Monteleoni, and Sarwate]{CMS11}
Kamalika Chaudhuri, Claire Monteleoni, and Anand~D. Sarwate.
\newblock Differentially private empirical risk minimization.
\newblock \emph{Journal of Machine Learning Research}, 12:\penalty0 1069--1109,
  2011.

\bibitem[Chaudhuri et~al.(2012)Chaudhuri, Sarwate, and Sinha]{CSS12}
Kamalika Chaudhuri, Anand~D. Sarwate, and Kaushik Sinha.
\newblock Near-optimal differentially private principal components.
\newblock In \emph{Advances in Neural Information Processing Systems}, pages
  998--1006, 2012.

\bibitem[Chen et~al.(2011)Chen, Mohammed, Fung, Desai, and Xiong]{PubSet11}
Rui Chen, Noman Mohammed, Benjamin~CM Fung, Bipin~C Desai, and Li~Xiong.
\newblock Publishing set-valued data via differential privacy.
\newblock In \emph{VLDB}, 2011.

\bibitem[De(2012)]{De11}
Anindya De.
\newblock Lower bounds in differential privacy.
\newblock In Ronald Cramer, editor, \emph{Theory of Cryptography}, volume 7194
  of \emph{Lecture Notes in Computer Science}, pages 321--338. Springer-Verlag,
  2012.

\bibitem[Dwork et~al.(2006{\natexlab{a}})Dwork, McSherry, Nissim, and
  Smith]{DMNS06}
C.~Dwork, F.~McSherry, K.~Nissim, and A.~Smith.
\newblock Calibrating noise to sensitivity in private data analysis.
\newblock In \emph{Theory of Cryptography}, 2006{\natexlab{a}}.

\bibitem[Dwork(2008)]{DS08}
Cynthia Dwork.
\newblock Differential privacy: A survey of results.
\newblock In \emph{Theory and Applications of Models of Computation}, pages
  1--19. Springer, 2008.

\bibitem[Dwork and Lei(2009)]{DL09}
Cynthia Dwork and Jing Lei.
\newblock Differential privacy and robust statistics.
\newblock In \emph{Proceedings of the 41st annual ACM symposium on Theory of
  computing}, pages 371--380. ACM, 2009.

\bibitem[Dwork et~al.(2006{\natexlab{b}})Dwork, Kenthapadi, McSherry, Mironov,
  and Naor]{Dwork06}
Cynthia Dwork, Krishnaram Kenthapadi, Frank McSherry, Ilya Mironov, and Moni
  Naor.
\newblock Our data, ourselves: Privacy via distributed noise generation.
\newblock In \emph{Advances in Cryptology-EUROCRYPT 2006}, pages 486--503.
  Springer, 2006{\natexlab{b}}.

\bibitem[Friedman and Schuster(2010)]{FS10}
A.~Friedman and A.~Schuster.
\newblock Data mining with differential privacy.
\newblock In \emph{KDD}, 2010.

\bibitem[Hardt and Rothblum(2010)]{SVT}
Moritz Hardt and Guy~N Rothblum.
\newblock A multiplicative weights mechanism for privacy-preserving data
  analysis.
\newblock In \emph{FOCS}, 2010.

\bibitem[Hardt and Talwar(2010)]{HT09}
Moritz Hardt and Kunal Talwar.
\newblock On the geometry of differential privacy.
\newblock In \emph{Proceedings of the 42nd ACM symposium on Theory of
  computing}, pages 705--714. ACM, 2010.

\bibitem[Jain et~al.(2012)Jain, Kothari, and Thakurta]{JKT12}
Prateek Jain, Pravesh Kothari, and Abhradeep Thakurta.
\newblock Differentially private online learning.
\newblock In \emph{COLT}, 2012.

\bibitem[Kasiviswanathan et~al.(2011)Kasiviswanathan, Lee, Nissim,
  Raskhodnikova, and Smith]{KLNRS08}
Shiva~Prasad Kasiviswanathan, Homin~K Lee, Kobbi Nissim, Sofya Raskhodnikova,
  and Adam Smith.
\newblock What can we learn privately?
\newblock \emph{SIAM Journal on Computing}, 40\penalty0 (3):\penalty0 793--826,
  2011.

\bibitem[Langford and McAllester(2004)]{shell}
John Langford and David McAllester.
\newblock Computable shell decomposition bounds.
\newblock \emph{J. Mach. Learn. Res.}, 5:\penalty0 529--547, 2004.

\bibitem[Li et~al.(2012)Li, Qardaji, Su, and Cao]{Privbasis12}
Ninghui Li, Wahbeh Qardaji, Dong Su, and Jianneng Cao.
\newblock Privbasis: frequent itemset mining with differential privacy.
\newblock In \emph{VLDB}, 2012.

\bibitem[McSherry and Talwar(2007)]{MT07}
Frank McSherry and Kunal Talwar.
\newblock Mechanism design via differential privacy.
\newblock In \emph{FOCS}, 2007.

\bibitem[Sarwate and Chaudhuri(2013)]{SC13}
A.D. Sarwate and K.~Chaudhuri.
\newblock Signal processing and machine learning with differential privacy:
  Algorithms and challenges for continuous data.
\newblock \emph{Signal Processing Magazine, IEEE}, 30\penalty0 (5):\penalty0
  86--94, Sept 2013.
\newblock ISSN 1053-5888.
\newblock \doi{10.1109/MSP.2013.2259911}.

\bibitem[Smith(2011)]{S11}
Adam Smith.
\newblock Privacy-preserving statistical estimation with optimal convergence
  rates.
\newblock In \emph{STOC}, 2011.

\bibitem[Smith and Thakurta(2013)]{ST13}
Adam Smith and Abhradeep Thakurta.
\newblock Differentially private feature selection via stability arguments, and
  the robustness of the lasso.
\newblock In \emph{COLT}, 2013.

\bibitem[Uhler et~al.(2012)Uhler, Slavkovic, and Fienberg]{Y13}
Caroline Uhler, Aleksandra~B. Slavkovic, and Stephen~E. Fienberg.
\newblock Privacy-preserving data sharing for genome-wide association studies.
\newblock arXiv:1205.0739, 2012.

\bibitem[Wasserman and Zhou(2010)]{WZ10}
Larry Wasserman and Shuheng Zhou.
\newblock A statistical framework for differential privacy.
\newblock \emph{Journal of the American Statistical Association}, 105\penalty0
  (489):\penalty0 375--389, 2010.

\bibitem[Zeng et~al.(2012)Zeng, Naughton, and Cai]{ondpfim12}
Chen Zeng, Jeffrey~F Naughton, and Jin-Yi Cai.
\newblock On differentially private frequent itemset mining.
\newblock In \emph{VLDB}, 2012.

\end{thebibliography}

\pagebreak

\appendix

\newcommand{\ha}{\hat{a}}
\newcommand{\alr}{\hat{a}^{\leq r}}
\newcommand{\alN}{\hat{a}^{\leq N}}
\newcommand{\mylog}[1]{\log \left(#1 \right)}
\newcommand{\calA}[2]{\mathcal{A}_{#1}(#2)}
\newcommand{\seqZ}{\boldsymbol Z_{1:N-1}}

\section{Privacy Analysis}
\label{app:privacy}

In this section, we present the proof of Theorem~\ref{thm:privacy}.
We rely on composition results for approximate differential privacy to
analyze the three parts of Algorithm~\ref{alg:aa}:
\begin{itemize}
  \item Differential privacy of releasing $m$ after Step~\ref{step:m}.

  \item Differential privacy of releasing $\ell$ after Step~\ref{step:ell}.

  \item Approximate differential privacy of releasing $I$ after
    Step~\ref{step:I}.

\end{itemize}
We make this explicit by encapsulating these parts in
Algorithm~\ref{alg:max} ($\algM$), Algorithm~\ref{alg:svt} ($\algS$), and
Algorithm~\ref{alg:argmax} ($\algA$), so we can write
Algorithm~\ref{alg:aa} as follows (after the definitions of $\order Tr$ and
$\order tr$):
\begin{enumerate}
  \item $m := \algM(\alpha/3,D)$.

  \item $\ell := \algS(\alpha/3,m,\order T1,\order T2,\dotsc,\order
    T{K-1},D)$.

  \item $I := \algA(\alpha/3,\ell,D)$.

\end{enumerate}

\subsection{$\max$ Estimation}

The first part of Algorithm~\ref{alg:aa} is a standard application of the
Laplace mechanism; it is detailed in Algorithm~\ref{alg:max}.

\begin{algorithm}[t]
  \caption{$\algM(\alpha,D)$}
  \label{alg:max}
  \begin{algorithmic}[1]
    \renewcommand{\algorithmicrequire}{\textbf{input}}
    \renewcommand{\algorithmicensure}{\textbf{output}}

    \REQUIRE Privacy parameter $\alpha > 0$, database $D \in \cX^n$.

    \ENSURE Max estimate $m \in \bbR$.

    \STATE Draw $Z \sim \Lap(1/\alpha)$.

    \RETURN $\order f1(D) + Z/n$.

  \end{algorithmic}
\end{algorithm}

\begin{lemma}[\cite{DMNS06}]
  \label{lem:max-privacy}
  $\algM(\alpha,\cdot)$ is $\alpha$-differentially private.
\end{lemma}

\begin{lemma}
  \label{lem:max-utility}
  With probability at least $1-\delta$,
  \begin{align*}
    \algM(\alpha,D)
    & \leq \order f1(D) + \frac{1}{n\alpha} \ln \frac{1}{2\delta} .
  \end{align*}
\end{lemma}
\begin{proof}
  This follows from the tail properties of the Laplace distribution.
\end{proof}

\subsection{Certifying the Margin Condition}

The second part of Algorithm~\ref{alg:aa} is an application of the ``sparse
vector technique'' to certify the margin condition; it is detailed in
Algorithm~\ref{alg:svt}.

\begin{algorithm}[t]
  \caption{$\algS(\alpha,m,\theta_1,\theta_2,\dotsc,\theta_{K-1},D)$}
  \label{alg:svt}
  \begin{algorithmic}[1]
    \renewcommand{\algorithmicrequire}{\textbf{input}}
    \renewcommand{\algorithmicensure}{\textbf{output}}

    \REQUIRE Privacy parameter $\alpha > 0$, max estimate $m \in \bbR$,
    thresholds $\theta_1, \theta_2, \dotsc, \theta_{K-1} \in \bbR$,
    database $D \in \cX^n$.

    \ENSURE Rank $r \in \{1,2,\dotsc,K\}$.

    \STATE Draw $G \sim \Lap(2/\alpha)$ and $Z_1, Z_2, \dotsc, Z_{K-1} \stackrel{\text{iid}}{\sim} \Lap(4/\alpha)$

    \FOR{$r=1,2,\dotsc,K-1$}

      \IF{$m - \order f{r+1}(D) > (Z_r + G) / n + \theta_r$}

        \RETURN $r$.

      \ENDIF

    \ENDFOR

    \RETURN $K$.

  \end{algorithmic}
\end{algorithm}

\begin{lemma}
  \label{lem:svt-privacy}
  For any $m,\theta_1,\theta_2,\dotsc,\theta_{K-1} \in \bbR$,
  $\algS(\alpha,m,\theta_1,\theta_2,\dotsc,\theta_{K-1},\cdot)$ is
  $\alpha$-differentially private.
\end{lemma}
\begin{proof}
  \newcommand{\PrG}{\ensuremath{{\Pr}_{|G}}}
This is an application of the sparse vector technique from~\cite{SVT}
that halts as soon as the first ``query'' is answered positively.
We give the privacy analysis for completeness.
For clarity, we suppress the dependence of $\algS$ on all inputs
except $D$, and define $\order F{r+1} := m - \order f{r+1} -
\theta_r$, which inherits the $(1/n)$-Lipschitz property from $\order
f{r+1}$.

Pick any neighboring datasets $D$ and $D'$, and pick any $\ell \in
\{1,2,\dotsc,K\}$.
We use the notation $\PrG(\cdot)$ for conditional probabilities where
the value of $G$ is fixed, so $\Pr(\cdot) = \bbE( \PrG(\cdot) )$,
where the expectation is taken with respect to $G$.
Observe that
\begin{equation}
  \label{eq:condprob}
  \PrG(\algS(D) = \ell)
  =
  \PrG(\algS(D) \leq \ell | \algS(D) > \ell-1)
  \prod_{r=1}^{\ell-1}
  \PrG(\algS(D) > r | \algS(D) > r-1)
  .
\end{equation}
From the definition of \algS and $\order F{r+1}$,
\begin{align*}
  \PrG(\algS(D) > r | \algS(D) > r-1)
  & = \PrG\parens{ \order F{r+1}(D) \leq \frac{Z_r + G}n }
  \quad \forall r \in \{1,2,\dotsc,\ell-1\} ,
\end{align*}
and
\begin{align*}
  \PrG(\algS(D) \leq \ell | \algS(D) > \ell-1)
  & = \PrG\parens{ \order F{\ell+1}(D) > \frac{Z_\ell + G}n }
  .
\end{align*}
Write $\vZ_{1:\ell-1} := (Z_1,Z_2,\dotsc,Z_{\ell-1})$, and define
for any $g \in \bbR$,
\[
  \cZ_g(D) := \braces{
    \vz \in \bbR^{\ell-1} :
    \order F{r+1}(D) \leq \frac{z_r + g}n
    \quad \forall r \in \{1,2,\dotsc,\ell-1\}
  }
  ,
\]
so that
\begin{align*}
  \prod_{r=1}^{\ell-1} \PrG(\algS(D) > r | \algS(D) > r-1)
  & = \prod_{r=1}^{\ell-1}
  \PrG\parens{ \order F{r+1}(D) \leq \frac{Z_r + G}n }
  \\
  & = \PrG\parens{ \vZ_{1:\ell-1} \in \cZ_G(D) } .
\end{align*}
Hence, substituting into \eqref{eq:condprob}, we have
\[
  \PrG(\algS(D) = \ell)
  =
  \PrG\parens{ \order F{\ell+1}(D) > \frac{Z_\ell + G}{n} }
  \PrG(\vZ_{1:\ell-1} \in \cZ_G(D))
  .
\]
Letting $p$ denote the density of $G$, we have the following chain of
inequalities:
\begin{align}
  \lefteqn{ \Pr(\algS(D) = \ell) = \bbE(\PrG(\algS(D) = \ell)) }
  \nonumber \\
  & = \int_{-\infty}^\infty
  \PrG\parens{ \order F{\ell+1}(D) > \frac{Z_\ell + g}{n} }
  \PrG(\vZ_{1:\ell-1} \in \cZ_g(D))
  p(g) dg
  \nonumber \\
  & \leq \exp(\alpha/2) \int_{-\infty}^\infty
  \PrG\parens{ \order F{\ell+1}(D) > \frac{Z_\ell + g}{n} }
  \PrG(\vZ_{1:\ell-1} \in \cZ_g(D))
  p(g+1) dg
  \label{eq:change-density} \\
  & = \exp(\alpha/2) \int_{-\infty}^\infty
  \PrG\parens{ \order F{\ell+1}(D) > \frac{Z_\ell + g-1}{n} }
  \PrG(\vZ_{1:\ell-1} \in \cZ_{g-1}(D))
  p(g) dg
  \nonumber \\
  & \leq \exp(\alpha/2) \int_{-\infty}^\infty
  \PrG\parens{ \order F{\ell+1}(D) > \frac{Z_\ell + g-1}{n} }
  \PrG(\vZ_{1:\ell-1} \in \cZ_g(D'))
  p(g) dg
  \label{eq:change-rterms} \\
  & \leq \exp(\alpha) \int_{-\infty}^\infty
  \PrG\parens{ \order F{\ell+1}(D') > \frac{Z_\ell + g}{n} }
  \PrG(\vZ_{1:\ell-1} \in \cZ_g(D'))
  p(g) dg
  \label{eq:change-lterm} \\
  & = \exp(\alpha) \Pr(\algS(D') = \ell)
  \nonumber .
\end{align}
To prove~\eqref{eq:change-density}, we use the fact $p(g) \leq
\exp(\alpha/2) p(g+1)$ since $p$ is the Laplace density with scale
parameter $\alpha/2$.
To prove~\eqref{eq:change-rterms}, observe that for all $r \in
\{1,2,\dotsc,\ell-1\}$,
the $(1/n)$-Lipschitz property of $\order F{r+1}$ implies
\[
  \order F{r+1}(D) \leq \frac{Z_r + g - 1}{n} \ \Longrightarrow \
  \order F{r+1}(D') \leq \frac{Z_r + g}{n}
  .
\]
This, in turn, implies $\cZ_{g-1}(D) \subseteq \cZ_g(D')$,
so~\eqref{eq:change-rterms} follows.
To prove~\eqref{eq:change-lterm}, we use the following.
Observe that
\begin{equation*}
  \order F{\ell+1}(D) > \frac{Z_\ell + g - 1}{n} \ \Longrightarrow \
  \order F{\ell+1}(D') > \frac{Z_\ell + g - 2}{n}
\end{equation*}
by the $(1/n)$-Lipschitz property of $\order F{\ell+1}$.
Therefore
\begin{align*}
  \PrG\parens{ \order F{\ell+1}(D) > \frac{Z_\ell + g-1}{n} }
  & \leq
  \PrG\parens{ \order F{\ell+1}(D') > \frac{Z_\ell + g-2}{n} }
  \\
  & \leq
  \exp(\alpha/2)
  \PrG\parens{ \order F{\ell+1}(D') > \frac{Z_\ell + g}{n} }
\end{align*}
where we use the fact that $Z_\ell \sim \Lap(\alpha/4)$ for the last
step, so~\eqref{eq:change-lterm} follows.
\end{proof}

\begin{lemma}
  \label{lem:svt-utility}
  With probability at least $1-\delta$, if
  $\algS(\alpha,m,\theta_1,\theta_2,\dotsc,\theta_{K-1},D) = r$ then
  \[
    m - \order f{r+1}(D)
    > \theta_r -\frac2{n\alpha} \ln \frac{1}{\delta}
    -\frac4{n\alpha} \ln \frac{r(r+1)}{\delta}
    .
  \]
\end{lemma}
\begin{proof}
  Using the tail bound for the Laplace distribution,
  \begin{equation*}
    \Pr\parens{ G < -\frac2{\alpha} \ln \frac{1}{\delta} }
    \leq \frac{\delta}{2}
  \end{equation*}
  and
  \begin{equation*}
    \Pr\parens{ Z_r < -\frac4{\alpha} \ln \frac{r(r+1)}{\delta} }
    \leq \frac{\delta}{2r(r+1)}
  \end{equation*}
  for each $r \in \{1,2,\dotsc,K-1\}$.
  Therefore, by a union bound, with probability at least $1-\delta$,
  \begin{equation*}
    G \geq -\frac2{\alpha} \ln \frac{1}{\delta} \quad\text{and}\quad
    Z_r \geq -\frac4{\alpha} \ln \frac{r(r+1)}{\delta}
    \ \forall r \in \{1,2,\dotsc,K-1\} .
  \end{equation*}
  The claim follows.
\end{proof}

\subsection{Restricted Exponential Mechanism}

The third part of Algorithm~\ref{alg:aa} uses the exponential
mechanism on the top $\ell$ items to select one of these items; it is
detailed in Algorithm~\ref{alg:argmax}.

\begin{algorithm}[t]
  \caption{$\algA(\alpha,\ell,D)$}
  \label{alg:argmax}
  \begin{algorithmic}[1]
    \renewcommand{\algorithmicrequire}{\textbf{input}}
    \renewcommand{\algorithmicensure}{\textbf{output}}

    \REQUIRE Privacy parameter $\alpha > 0$, number of items $\ell > 0$,
    database $D \in \cX^n$.

    \ENSURE Item $I \in \cU$.

    \STATE Let $\cU_\ell$ be the set of $\ell$ items in $\cU$ with
    highest $f(i,D)$ value, ties broken arbitrarily.

    \STATE Draw $I \sim \vp$ where $p_i \propto \ind{i \in \cU_\ell}
    \exp(n\alpha f(i,D)/2)$.

    \RETURN $I$.

  \end{algorithmic}
\end{algorithm}

\begin{lemma}
  \label{lem:expmech}
  Assume $D$ satisfies the $(\ell,\gamma)$-margin condition with
  \[
    \gamma \geq \frac2n \parens{ 1 + \frac{\ln(\ell/\beta)}{\alpha} }
    .
  \]
  Then for any neighbor $D' \in \cX^n$ of $D$, and any $S \subseteq
  \cU$,
  \[
    \Pr(\algA(\alpha,D) \in S)
    \leq \exp(\alpha) \cdot \Pr(\algA(\alpha,D') \in S) + \beta
    .
  \]
\end{lemma}
\begin{proof}
  For any $r \in \{1,2,\dotsc,K\}$ and dataset $\tilde{D} \in \cX^n$,
  let $H_{\tilde{D}} \subseteq \cU$ denote the $r$ items of highest
  $f(\cdot,\tilde{D})$ value (ties broken arbitrarily).
  (In Algorithm~\ref{alg:argmax}, we have $\cU_\ell = H_D$.)
  It suffices to show that
  \[
    \Pr(\algA(\alpha,\ell,D') = i)
    \leq \max\braces{
      \Pr(\algA(\alpha,\ell,D) = i) \exp(\alpha) ,\ \beta/\ell
    }
    , \quad \forall i \in H_{D'}
    .
  \]
  This is because $\Pr(\algA(\alpha,\ell,D') \notin H_{D'}) = 0$ and
  $|H_{D'}| = \ell$.

  Fix any $i \in H_{D'}$.
  Because $f(j,\cdot)$ is $(1/n)$-Lipschitz for every $j \in \cU$, so
  is $\order fr(\cdot)$ for every $r \in [K]$.
  Therefore
  \[
    \sum_{r=1}^\ell
    \exp\parens{\frac{n\alpha}{2} \order fr(D')}
    \geq
    \sum_{r=1}^\ell
    \exp\parens{\frac{n\alpha}{2} \order fr(D)}
    \exp(-\alpha/2)
    .
  \]
  Also by the $(1/n)$-Lipschitz property,
  \[
    \exp\parens{\frac{n\alpha}{2} f(i,D')}
    \leq
    \exp\parens{\frac{n\alpha}{2} f(i,D)}
    \exp(\alpha/2)
    .
  \]
  Therefore, combining the two displayed equations above gives
  \begin{equation}
    \Pr(\algA(\alpha,\ell,D') = i)
    = \frac{\exp\parens{\frac{n\alpha}{2}f(i,D')}}
    {\sum_{r=1}^\ell \exp\parens{\frac{n\alpha}{2}\order fr(D')}}
    \leq \frac{\exp\parens{\frac{n\alpha}{2}f(i,D)}}
    {\sum_{r=1}^\ell \exp\parens{\frac{n\alpha}{2}\order fr(D)}}
    \exp(\alpha)
    .
    \label{eq:expmech}
  \end{equation}
  If $i \in H_D$, then~\eqref{eq:expmech} reads
  \[
    \Pr(\algA(\alpha,\ell,D') = i) \leq \Pr(\algA(\alpha,\ell,D) = i)\exp(\alpha)
    .
  \]
  If $i \notin H_D$, then the assumption that $D$ satisfies the
  $(\ell,\gamma)$-margin condition implies
  \[
    f(i,D) \leq \order f1(D) - \gamma
    ;
  \]
  so combining the above inequality with~\eqref{eq:expmech}, as well
  as the assumption $\gamma \geq (2/n)(1+\ln(\ell/\beta)/\alpha)$,
  gives
  \[
    \Pr(\algA(\alpha,\ell,D') = i) \leq
    \frac
    {\exp\parens{\frac{n\alpha}{2}\parens{\order f1(D) - \gamma}}}
    {\exp\parens{\frac{n\alpha}{2}\order f1(D)}}
    \exp(\alpha)
    \leq \beta/\ell
    .
    \qedhere
  \]
\end{proof}

\subsection{Privacy of Algorithm~\ref{alg:aa}}
\newcommand\AND{\ensuremath{\,\wedge\,}}
\newcommand\GIVEN{\ensuremath{\,\vert\,}}

For clarity, we suppress the privacy parameter inputs to the
algorithms.
By standard composition results for differential privacy~\cite{DMNS06},
Lemma~\ref{lem:max-privacy}, and Lemma~\ref{lem:svt-privacy}, the release
of $\algM(D)$ and $\algS(\algM(D),D)$ is $(2\alpha/3)$-differentially
private.
Define the shorthand $\algMS(D) := (\algM(D),\algS(\algM(D),D))$, and
let $\mu_D$ denote the corresponding probability measure over the
range of $\algMS(D)$.

For a dataset $D \in \cX^n$, let $\cV_D$ be set of $(\tilde{m},\tilde\ell)$
pairs (i.e., possible outputs of $\algMS$) such that
\[
  \tilde{m} \leq \order f1(D) + \frac{3}{n\alpha} \ln \frac{3}{2\delta}
  \quad\text{and}\quad
  \tilde{m} - \order f{\tilde\ell+1}(D)
  > \order T{\tilde\ell}
  -\frac{12}{n\alpha} \ln \frac{3\tilde\ell(\tilde\ell+1)}{\delta}
  -\frac{6}{n\alpha} \ln \frac3\delta
  .
\]

If $(m,\ell) \in \cV_D$, then the values of $\order T\ell$ and $\order
t\ell$ certify that $D$ satisfies the $(\ell,\order t\ell)$-margin
condition.
Lemma~\ref{lem:max-utility} and Lemma~\ref{lem:svt-utility} imply that
\[
  \mu_D(\cV_D) \geq 1-\frac{2\delta}3 .
\]
Also, observe that if $\beta := \delta \exp(-2\alpha/3) / 3$, then
\[
  \order t\ell
  = \frac{2}{n} \parens{
    1 + \frac{\ln(\ell/\beta)}{\alpha/3}
  }
  .
\]
Therefore, for any neighbor $D' \in \cX^n$ of $D$, and any $S
\subseteq \cU$,
\begin{align*}
  \Pr( \algAA(D) \in S )
  & = \int
  \Pr( \algA(\ell,D) \in S \GIVEN \algMS(D) = (m,\ell))
  d\mu_D
  \\
  & \leq \int_{\cV_D}
  \Pr( \algA(\ell,D) \in S \GIVEN \algMS(D) = (m,\ell))
  d\mu_D
  + \frac{2\delta}3
  \\
  & \leq \int_{\cV_D}
  \parens{
    e^{\alpha/3}
    \Pr( \algA(\ell,D') \in S \GIVEN \algMS(D) = (m,\ell))
    + \beta
  }
  e^{2\alpha/3}
  d\mu_{D'}
  + \frac{2\delta}3
  \\
  & = \int_{\cV_D}
  \parens{
    e^{\alpha/3}
    \Pr( \algA(\ell,D') \in S \GIVEN \algMS(D') = (m,\ell))
    + \frac{\delta e^{-2\alpha/3}}3
  }
  e^{2\alpha/3}
  d\mu_{D'}
  + \frac{2\delta}3
  \\
  & \leq \int
  \parens{
    e^{\alpha/3}
    \Pr( \algA(\ell,D') \in S \GIVEN \algMS(D') = (m,\ell))
    + \frac{\delta e^{-2\alpha/3}}3
  }
  e^{2\alpha/3}
  d\mu_{D'}
  + \frac{2\delta}3
  \\
  & =
  e^{\alpha}
  \Pr( \algAA(D') \in S )
  + \delta
  .
\end{align*}
Above, the second inequality follows from Lemma~\ref{lem:expmech} and
the $(2\alpha/3)$-differential privacy of $\algMS$.
\qed

\section{Utility Analysis}
\label{app:utility}

\begin{proof}[Proof of Theorem~\ref{thm:utility}]
  Using tail bounds for the Laplace distribution, it follows that with
  probability at least $1-\eta/2$,
  \begin{equation*}
    Z \geq -\frac{3}{\alpha} \ln \frac{3}{\eta} , \quad
    G \leq \frac{6}{\alpha} \ln \frac{3}{\eta} , \quad
    Z_{\ell^*} \leq \frac{12}{\alpha} \ln \frac{3}{\eta}
    .
  \end{equation*}
  In this event, the assumption that $D$ satisfies the
  $(\ell^*,\gamma^*)$-margin condition implies that
  \[
    \parens{ \order f1(D) + Z/n }
    - \order f{\ell^*+1}(D) > (Z_{\ell^*} + G) / n + \order T{\ell^*}
    ,
  \]
  so the while-loop terminates with $\ell \leq \ell^*$.
  Also, the probability distribution $\vp$ in Step~\ref{step:p} of
  Algorithm~\ref{alg:aa} assigns probability mass at most $\eta/2$ to
  the set of items $i$ with
  \[
    f(i,D) \leq \order f1(D) - \frac{6\ln(2\ell/\eta)}{n\alpha}
    .
  \]
  Therefore, by a union bound, the item $I$ returned by
  Algorithm~\ref{alg:aa} satisfies
  \[
    f(I,D) > \order f1(D) - \frac{6\ln(2\ell^*/\eta)}{n\alpha}
  \]
  with probability at least $1-\eta$.
\end{proof}

\section{Proofs of Lower Bounds}
\label{app:proofs}

\begin{proof}[Proof of Theorem~\ref{thm:lb2}]
  We construct the private maximization problem as follows.
  Let the domain $\cX := 2^\cU$ (subsets of items), and define $f :
  \cU \times \cX^n \to \bbR$ by
  \begin{align*}
    f(i,D)
    & := \frac1n \sum_{s=1}^n \ind{i \in D_s}
    .
  \end{align*}
  In other words, the function $f(i,\cdot)$ is the fraction of entries
  containing $i$.
  It is easy to see that $f(i,\cdot)$ is $(1/n)$-Lipschitz for all $i
  \in \cU$.

  Let $m := \min\{n/2, \log((\ell-1)/2)/\alpha\}$.
  We define a collection of $\ell$ datasets $D^1, D^2, \dotsc, D^\ell
  \in \cX^n$ with the following properties:
  \begin{enumerate}
    \item
      For each $i$, the first $n/2$ entries of $D^i$ are equal to
      $[\ell] := \{1,2,\dotsc,\ell\}$, the next $n/2 - m$ are equal of
      $D^i$ are equal to $\emptyset$, and the last $m$ entries of
      $D^i$ are equal to $\{i\}$.
      Therefore
      \[
        f(j,D^i) =
        \begin{cases}
          0 & \text{if $j \notin [\ell]$} , \\
          \frac12 & \text{if $j \in [\ell]\setminus\{i\}$} , \\
          \frac12 + \frac{m}{n} & \text{if $j = i$} ,
        \end{cases}
      \]
      so $f(i,D^i) = \order f1(D^i)$ and $D^i$ satisfies the
      $(\ell,m/n)$-margin condition.

    \item
      For each $i \neq j$, the datasets $D^i$ and $D^j$ differ only in
      (the last) $m$ entries.

  \end{enumerate}

  Let $\cA$ be $(\alpha,\delta)$-approximate differentially private.
  Assume for sake of contradiction that
  \[
    \Pr\parens{ f(\cA(D^i), D^i) > \order f1(D^i) - \frac{m}{n} }
    \geq \frac{1}{2}
  \]
  for all $i \in [\ell]$.
  Since only $i$ satisfies $f(i,D^i) > \order f1(D^i) - m/n$, this is
  the same as $\Pr( \cA(D^i) = i ) \geq 1/2$ for all $i \in [\ell]$.
  This then implies the following chain of inequalities leading to a
  contradiction:
  \begin{align*}
    \frac{1}{2} & > \Pr(\cA(D^i) \neq i) \\
    & \geq \sum_{j \in [\ell]\setminus\{i\}} \Pr(\cA(D^i) = j) \\
    & \geq \sum_{j \in [\ell]\setminus\{i\}} e^{-\alpha m} \Pr(\cA(D^j) =
    j) - \frac{\delta}{1-e^{-\alpha}} \\
    & \geq (\ell-1) \parens{ \frac{e^{-\alpha m}}{2} - \frac{\delta}{1-e^{-\alpha}} }
    \geq \frac12
    .
  \end{align*}
  The first inequality above is by assumption;
  the third inequality follows from Lemma~\ref{lem:adp};
  the fourth inequality again uses the assumption;
  and the final inequality follows by the definition of $m$ and the
  condition on $\delta$.
  Since a contradiction is reached, there must exist some $i \in
  [\ell]$ such that $\Pr( f(\cA(D^i), D^i) > \order f1(D^i) - m/n ) <
  1/2$.
\end{proof}

\begin{lemma}[\cite{CH12}] \label{lem:adp}
  Let $D$ and $D'$ be any two datasets that differ in at most $k$
  entries, and let $\cA$ be any $(\alpha, \delta)$-approximate
  differentially private algorithm with range $\cS$.
  Then, for any $S \subseteq \cS$,
  \[
    \Pr(\cA(D) \in S) \geq e^{-k \alpha} \Pr(\cA(D') \in S) -
    \frac{\delta}{1-e^{-\alpha}}
    .
  \]
\end{lemma}

\end{document}